\documentclass[authoryear]{eshbaugh}

\newcommand{\citep}{\parencite}
\newcommand{\citet}{\textcite}

\title{Align Before You Combine: Reference Space Calibration for Supervision Without Ground Truth}

\shorttitle{Align Before You Combine}

\author{Jackson Eshbaugh}
\affiliation{Department of Computer Science\\Lafayette College\\Easton, PA, USA}
\email{eshbaugj@lafayette.edu}
\orcid{0009-0009-1806-2166}

\author{Jorge Silveyra}
\affiliation{Department of Computer Science\\Lafayette College\\Easton, PA, USA}
\email{silveyj@lafayette.edu}
\orcid{0009-0004-9769-7250}

\publicationstatus{Preprint}
\version{1.0}
\preprintdate{October 2026}

\begin{document}

\maketitle

\begin{abstract}
We introduce a calibration-first framework that produces supervision scores without access to ground-truth labels or a shared annotation space.
Our framework aligns subset-specific scorers using a synthetic ordinal reference space before fusion.
This reference space is constructed from ordered calibration features that represent the latent concept, providing a common scale on which otherwise incomparable scorer outputs can be aligned.
Because our calibration procedure uses the reference space rather than training samples, it is independent of the training set’s empirical distribution.
Across three benchmark datasets, our framework consistently outperforms uncalibrated averaging and achieves higher primary-metric point estimates on the evaluation metrics than the best individual scorer. Performance relative to sample-dependent baselines varies by domain, with absolute differences below 0.02 on Ames Housing and below 0.01 on Breast Cancer Wisconsin and Wine Quality. 
After Bonferroni correction, differences remain significant for all three comparisons on Ames Housing and one on Breast Cancer Wisconsin.
Additionally, we show that using fewer calibration levels per feature can closely approximate higher-resolution results at substantially lower computational cost. Together, these results support our framework as a viable approach to construct supervision scores when neither ground-truth labels nor a shared annotation space is available.
\end{abstract}

\section{Introduction}
\label{sec:intro}

Heterogeneous evidence sources are often used to train machine learning systems, yet they frequently lack ground-truth labels or a shared annotation space. Current approaches, such as weak supervision and multimodal fusion, typically assume that supervision sources operate within a common target space. This assumption is problematic when supervision is derived from heterogeneous subset-specific scorers that operate on different evidence representations and incomparable numerical scales. In such settings, each source may act as a noisy semantic observer of an underlying latent concept, where the resulting signals may exist in different scales and in different representational spaces. Consequently, the outputs cannot be directly compared and require calibration before meaningful fusion can occur. 

In this document, we introduce a calibration-first framework for supervision without ground truth that aligns subset-specific scorers using a synthetic ordinal reference space before fusion. Our framework uses this reference space as a semantic calibration mechanism. By evaluating subset-specific scorers relative to this shared reference space, heterogeneous evidence sources can be aligned onto a common semantic scale prior to fusion. The resulting aligned scores are then fused into a supervision score. Ultimately, this approach enables supervision from heterogeneous sources where neither labeled data nor a shared annotation space is available.

To assess the viability of our framework, we conducted a case study in a residential energy-retrofit application---where text-derived and structured scorers operate on different numerical scales---alongside benchmarks from the housing and medical domains. Our contributions are: (1) a calibration-first framework for supervision without ground truth; (2) a synthetic ordinal reference space for aligning heterogeneous scorers prior to fusion; and (3) formal properties and empirical evidence demonstrating ordinal preservation and improved supervision quality.

\section{Related Work}
\label{sec:rw}
Given the diverse scope of this problem and to contextualize the landscape in which our work lies, we review three related research areas that address different challenges. Specifically, (1) weak supervision focuses on learning from imperfect or noisy labels, (2) multimodal learning addresses the integration of heterogeneous data representations, and (3) latent construct research studies how unobservable concepts can be inferred from indirect sources. 

Modern weak supervision research has given rise to programmatic weak supervision (PWS), which combines multiple distant supervision approaches through noisy labeling functions and latent label modeling~\citep{zhangSurveyProgrammaticWeak2022}. Current PWS involves users encoding domain knowledge into noisy, explicitly engineered labeling functions that collectively enable latent-label estimation without direct supervision and generally assumes these functions produce values within a shared target space~\citep{Ratner2017,cachay2021endtoend,boecking2021interactive,pmlr-v151-yu22c}. Recent work has extended the PWS paradigm by constructing probabilistic supervision from multiple imperfect diagnostic tests or expert assessments through latent class modeling~\citep{yazbeckLatentSupervisionMethod2026}, and by exploring constraint-based alternatives to latent generative label models~\citep{ICLR2025_e91fb65c}. These approaches still assume supervision sources operate within a common label space and focus on recovering latent labels rather than aligning heterogeneous scoring functions defined on incomparable semantic scales. 

A related line of research, multimodal machine learning, concentrates on integrating heterogeneous sources such as text, structured measurements, and simulation outputs. Most multimodal systems assume ground-truth supervision and treat modalities as complementary predictive features aligned to a shared target, combining them through fusion strategies to improve downstream performance~\citep{ngiam2011multimodal, Baltrusaitis2019, Liang2024}. Recent developments have also explored calibration mechanisms for combining heterogeneous information sources, demonstrating that alignment of subset-specific signals prior to fusion can improve downstream retrieval and reasoning performance~\citep{bacellarCalibratedFusionHeterogeneous2026}. These models are generally trained using observable ground truth targets.

A third area of work is related to measuring latent constructs---constructs that are not directly observable. In fields such as education, psychology, and other social sciences, psychometrics is used to measure latent constructs, such as intelligence, personality traits, or educational proficiency. Classical test theory and Item Response Theory (IRT) model observed responses as imperfect indicators of an underlying latent variable and provide mechanisms for placing heterogeneous observations onto a common measurement scale~\citep{furrPsychometrics2008,embretsonItemResponseTheory2025,raschProbabilisticModelsIntelligence1980}. Construct validity emphasizes that meaningful measurement requires a theoretically justified relationship between observed indicators and the latent construct being measured~\citep{cronbachConstructValidityPsychological1955,clarkConstructingValidityNew2019}. 

The existing literature provides several foundations for the problem considered in this work, but none directly address supervision construction from heterogeneous evidence sources when no common target space or ground-truth labels are available. Specifically, our work expands beyond prior research in three key ways:

\begin{enumerate}
    \item \textbf{Beyond a Shared Space}: Existing weak supervision frameworks assume that supervision sources already operate within a shared target space. In contrast, we construct a common semantic scale on which heterogeneous subset-specific scorers can be meaningfully compared.

    \item \textbf{Beyond Multimodal Prediction}: Multimodal machine learning typically assumes an observable target variable and uses modalities as complementary predictive features. Our goal is not to improve prediction through fusion but to construct supervision itself, aligning subset-specific observations when no ground-truth target is available.

    \item \textbf{Beyond Latent Construct Estimation}: Psychometrics addresses the problem of measuring latent constructs through the use of multiple imperfect indicators and provides mechanisms for placing heterogeneous observations onto a common measurement scale. Instead, our approach constructs an external reference structure composed of synthetic examples enabling supervision without ground-truth labels or a pre-existing measurement instrument.
\end{enumerate}

 We address the challenge of constructing supervision from heterogeneous subset-specific scorers that operate on incomparable semantic scales in the absence of ground-truth labels using a calibration-first framework that uses a synthetic semantic reference space to align heterogeneous scorers before fusion.

\section{Framework}
\label{sec:framework}

Our framework consist of two operational phases: (1) calibration using a synthetic ordinal reference space and (2) scoring and fusion of unlabeled instances\footnote{Our code is available on GitHub: \url{https://github.com/Lafayette-EshbaughSilveyra-Group/calibration-first}}. These phases are depicted in Figure~\ref{fig:framework}. We first introduce the necessary preliminaries before describing each phase.

\begin{figure*}[t]
    \centering

    \begin{subfigure}[t]{0.48\textwidth}
        \centering
        \includegraphics[width=\linewidth]{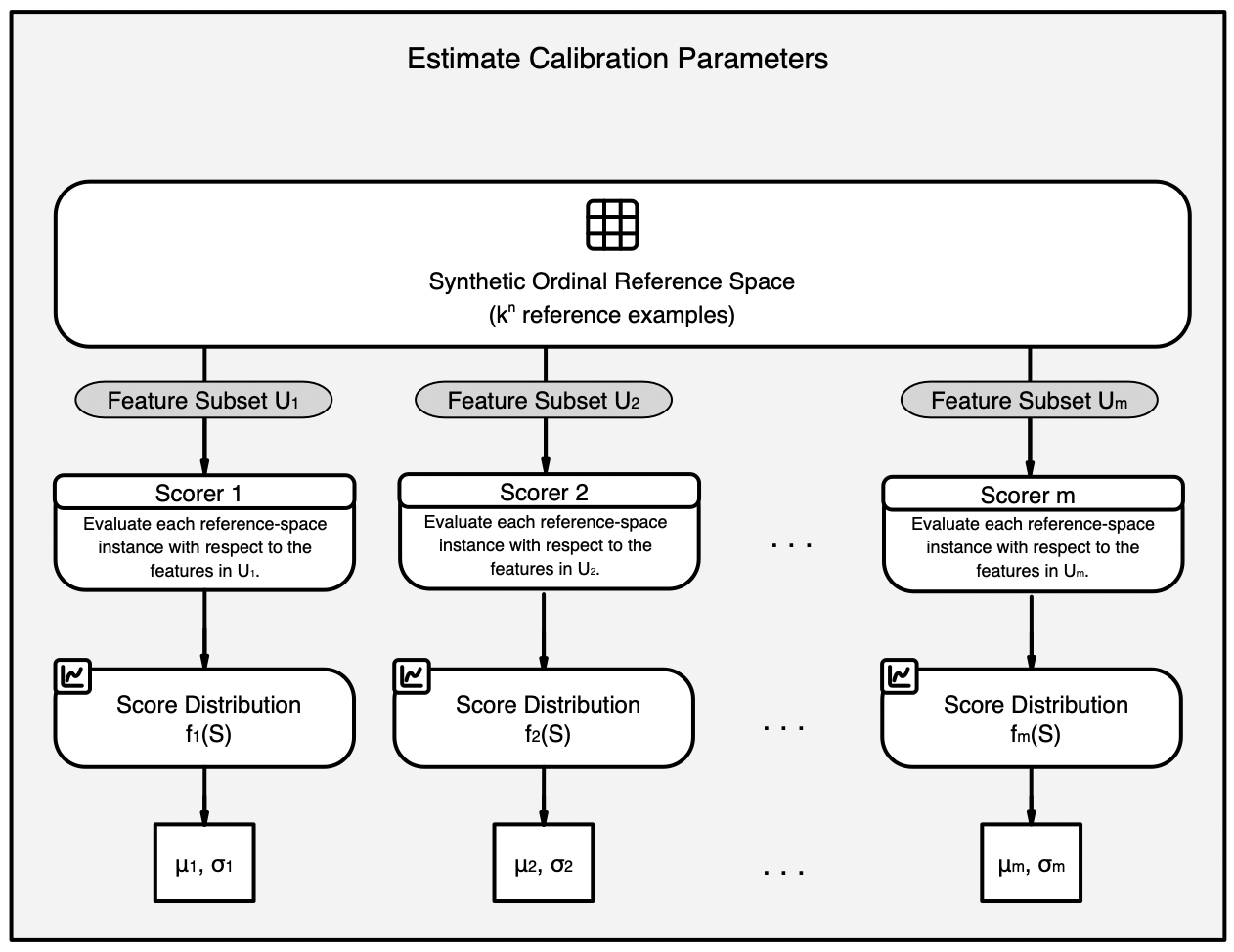}
        \caption{Calibration. Subset-specific scorers are evaluated over the synthetic ordinal reference space to estimate normalization statistics \((\mu_j,\sigma_j)\).}
        \label{fig:framework_calibration}
    \end{subfigure}
    \hfill
    \begin{subfigure}[t]{0.48\textwidth}
        \centering
        \includegraphics[width=\linewidth]{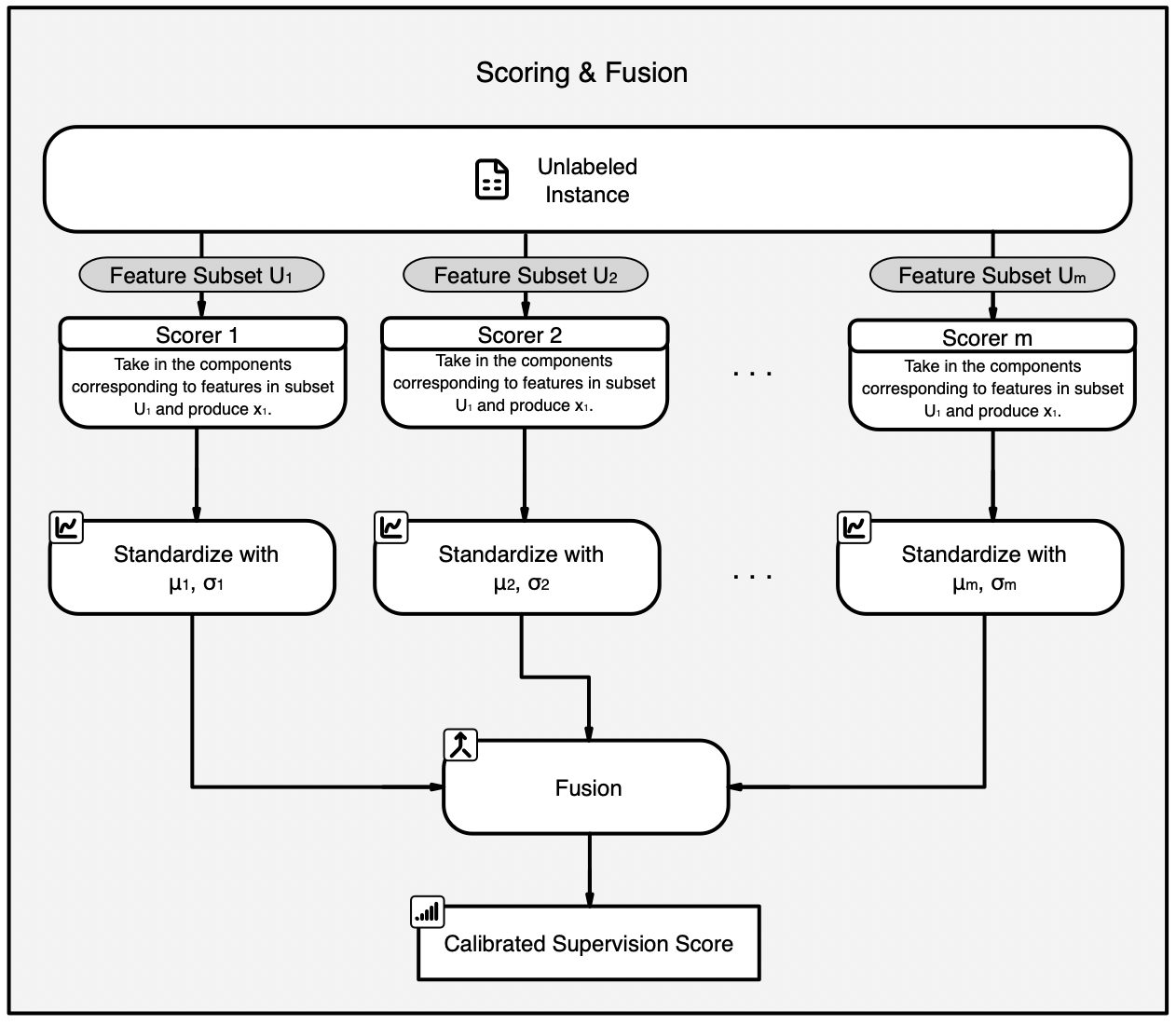}
        \caption{Scoring and fusion. Scores from a new unlabeled instance are standardized using the normalization statistics and fused into the supervision score \(\kappa\).}
        \label{fig:framework_fusion}
    \end{subfigure}

    \caption{
    Calibration-first supervision framework. During calibration, subset-specific scorers are evaluated over a synthetic ordinal reference space to estimate normalization statistics. During inference, scores for a new unlabeled instance are standardized using these statistics and fused into a calibrated supervision score \( \kappa \).
    }
    \label{fig:framework}
\end{figure*}

\subsection{Preliminaries}
\label{sec:framework:prelims}

\subsubsection{Notation and Assumptions}
\label{sec:framework:prelims:notation_assmptn}

We assume that the latent concept is characterized by \( n \) calibration features, contained in the set \( \mathcal{F} \). For each feature \( i \in \mathcal{F} \), let its set of possible values be denoted by \( A_i=\{a \mid b_i \preceq a \preceq c_i\}, \) where \( b_i \) and \( c_i \) denote the minimum and maximum values for feature \( i \), respectively, and \( \preceq \) denotes the feature's natural ordering. The Cartesian product \( \mathcal{A} = \prod_{i=1}^{n} A_i \) defines the ambient space of all possible configurations of the calibration features, where each element of \( \mathcal{A} \) is an ordered \( n \)-tuple specifying one value for every calibration feature. Any observed dataset satisfies \( \mathcal D\subseteq\mathcal A \), the synthetic ordinal reference space \( S \) satisfies \( S\subseteq\mathcal A \), and any new unlabeled instance satisfies \( x\in\mathcal A \).

Another foundational assumption of the framework is that each calibration feature is monotonically related to the latent concept of interest. Although observations of these features through different evidence sources may be noisy, incomplete, or expressed on incomparable scales, movement along the ordered levels of any calibration feature corresponds to consistent movement along the latent concept. Under this assumption, the synthetic reference space induces a meaningful ordinal structure over the latent concept, and subset-specific observations can be interpreted as noisy evidence of an instance's position within that structure. The monotonicity assumption defines a validity condition of the framework. If it is violated, the induced ordinal reference structure need not correspond to the latent concept, and no semantic validity guarantee is made for the resulting supervision score.

\subsubsection{Subset-Specific Scorers}
\label{sec:framework:prelims:scoring}

It is necessary to estimate an instance's position along the latent concept. The framework uses a collection of subset-specific scorers to determine this from heterogeneous observations. Each subset \( U \subseteq \mathcal{F} \) is selected a priori to represent a coherent portion of the latent concept. Its corresponding scorer is a function: 
\[
F_U : \prod_{i \in U} A_i \to \mathbb{R}
\]

A subset may correspond to a conventional data modality, such as text, images, or simulation output, but may also consist of a coherent group of related features or a concept-specific view of a shared evidence source. Accordingly, multiple features may be evaluated jointly by one scorer, and multiple scorers may operate on the same underlying observation to estimate different aspects of the latent concept.

Scorers may be machine learning-based or heuristic. A \textit{machine learning scorer} uses a learned model, such as a decision tree, multilayer perceptron, or large language model, to map the evidence associated with its designated feature subset to a scalar estimate. A \textit{heuristic scorer} deterministically computes a scalar estimate using a predefined function derived from expert knowledge or application-specific rules, such as linear scoring rules, deterministic scaling functions, or rule-based systems.

These scorers are used throughout both operational phases of the framework. During calibration, each scorer is evaluated over the evidence generated from the synthetic ordinal reference space introduced in Section~\ref{sec:framework:phases:calibration} to estimate scorer-specific calibration statistics (Figure~\ref{fig:framework_calibration}). During inference, the same scorers are applied to new unlabeled instances, and their outputs are standardized and fused to produce the final supervision score (Figure~\ref{fig:framework_fusion}).

The framework accommodates any combination of machine learning or heuristic scorers, provided that each produces a real-valued output whose ordering is consistent with the latent concept. To ensure meaningful aggregation across subset-specific scorers, scorer outputs must also share a common directionality, such that larger values consistently indicate greater correspondence to the latent concept or vice versa.

\subsection{Operational Phases}
\label{sec:framework:phases}
Given the monotonicity assumption and a collection of subset-specific scorers, the framework proceeds in two operational phases. The first evaluates each scorer over evidence generated from a synthetic ordinal reference space to estimate scorer-specific calibration statistics. The second applies the same scorers to previously unseen instances, standardizes their outputs using the corresponding calibration statistics, and fuses the standardized scores to produce a supervision score.

\subsubsection{Calibration via a Synthetic Reference Space}
\label{sec:framework:phases:calibration}

Calibration is performed using a synthetic ordinal reference space that defines a shared semantic reference for the latent concept across all subset-specific scorers. Because \( S \) is defined independently of the observed dataset, calibration is anchored to a domain-defined semantic range rather than to the empirical distribution of a particular sample. Beginning with the \( n \) calibration features introduced in Section~\ref{sec:framework:prelims:notation_assmptn}, each feature is discretized into \( k \) ordered levels spanning an observable range. These levels must preserve the assumed monotonic relationship between the feature and the latent concept, but need not be uniformly spaced. Collectively, they define the synthetic ordinal reference space used for calibration.

For a given application, the calibration levels may be obtained by partitioning each calibration feature into \( k \) ordered bins spanning its observable range. Let \( \mathcal{L}=\{L_1,L_2,\dots,L_n\} \) denote the collection of calibration-level sets, where each
\( L_i=\{\ell_{i1},\ell_{i2},\dots,\ell_{ik}\} \) contains the \( k \) ordered calibration levels associated with feature \( i \). Each calibration-level set is a subset of the corresponding feature domain; that is, \( L_i\subseteq A_i \). Calibration levels may be defined using domain knowledge, bin midpoints, empirical means, or other application-specific procedures. The framework is agnostic to this choice, requiring only that the resulting levels preserve the monotonic ordering assumed between the calibration feature and the latent concept.

The synthetic ordinal reference space is constructed as the Cartesian product of the calibration-level sets:
\[
S=\prod_{i=1}^{n}L_i.
\]
By construction, \( S\subseteq\mathcal A \). Each element \( s\in S \) represents a calibration point corresponding to a particular combination of ordinal feature levels. Collectively, these points span the ordinal range of the latent concept and provide a shared reference space for aligning scorer outputs.

For each calibration point, the evidence required by each subset-specific scorer is constructed from the corresponding feature configuration. Depending on the scorer, this evidence may take the form of a conventional modality, a derived representation, or a coherent group of calibration features. The resulting evidence is then evaluated using the same subset-specific scorer that will later be applied to unlabeled instances. This ensures that each scorer's calibration statistics are estimated from outputs produced by the scorer itself over its corresponding evidence representation.

In the worst case, calibration requires evaluating all \( |S|=k^n \) reference points with each of the \( m \) subset-specific scorers to estimate scorer-specific calibration statistics \( (\mu_j,\sigma_j) \), where \( j\in\{1,\dots,m\} \) indexes the scorers. If multiple calibration points produce the same evidence representation for a scorer, the corresponding evaluation may be cached and reused, reducing the number of distinct scorer evaluations without changing the reference distribution.

If scorer \( j \) has evaluation cost \( c_j \), the worst-case total calibration cost scales as \( O\left(k^n\sum_{j=1}^{m}c_j\right). \)
This cost is incurred only once during calibration to estimate the normalization statistics. Thereafter, scoring a new instance requires evaluating each subset-specific scorer once, yielding a complexity of \( O\left(\sum_{j=1}^{m}c_j\right). \)

The choice of \( k \) governs a tradeoff between calibration resolution and computational cost. Larger values of \( k \) produce a finer discretization of the calibration features, potentially yielding more representative normalization statistics, but increase the size of the reference space exponentially. In practice, \( k \) should be chosen to adequately resolve the ordinal structure of the latent concept while remaining computationally tractable given the cost of evaluating each scorer.

These design choices result in a calibration procedure that is not necessarily tied to a particular dataset. By constructing the reference space with this approach, we obtain the following benefits:

\begin{itemize}
\item \textbf{Comparability:} A calibrated value has the same intended interpretation across datasets.
\item \textbf{Stability:} Changes in sample composition do not themselves change the calibration mapping.
\item \textbf{Transferability:} Calibration learned from the synthetic reference space can be reused on new datasets rather than recomputed from each empirical distribution.
\item \textbf{Cleaner separation:} Changes in supervision scores reflect scorer responses to examples, rather than changes in the distribution used to normalize the scorers.
\end{itemize}

\subsubsection{Supervision Score Fusion \& Normalization}
\label{sec:framework:phases:fusion}

Once the calibration phase is complete, a new unlabeled instance \( x\in\mathcal A \) is evaluated by each subset-specific scorer, producing raw scores
\[
\{s_j(x)\}_{j=1}^{m},
\]
where \( j\in\{1,\dots,m\} \) indexes the scorers. Directly combining these raw scores risks privileging scorers with larger numerical ranges. Each score is therefore standardized using its corresponding calibration statistics \( \mu_j \) and \( \sigma_j \). The standardized output of scorer \( j \) is given by Equation~\ref{eq:standardized}, where \( \mu_j \) and \( \sigma_j \) are obtained as described in Section~\ref{sec:framework:phases:calibration}.

\begin{equation}
\label{eq:standardized}
\hat{s}_j(x)
=
\frac{s_j(x)-\mu_j}{\sigma_j}
\end{equation}

This transformation ensures that each subset-specific scorer's contribution to the fused signal reflects its relative position within the corresponding reference distribution rather than its original numerical scale. Standardization therefore aligns differences in scorer location and positive linear scale while preserving the ordering of scorer outputs. Because the transformation has positive slope \( 1 / \sigma_j \) when \( \sigma_j > 0 \), scorers exhibiting zero variance across the calibration space are excluded from fusion because they provide no discriminative information regarding the latent concept. The ordinal information provided by any monotone scorer is therefore unchanged by standardization. This property is formalized in Proposition~\ref{prop:ordinal_preservation}, proved in Appendix~\ref{app:props_proofs}.

The fused supervision score \( \kappa(x) \) is computed as the arithmetic mean of the standardized scorer outputs according to Equation~\ref{eq:overall_lab}.

\begin{equation}
\label{eq:overall_lab}
\kappa(x)
=
\frac{1}{m}
\sum_{j=1}^{m}
\hat{s}_j(x).
\end{equation}

In the absence of ground truth, we adopt equal weighting across the predefined subset-specific scorers as a neutral default because no empirical basis exists for estimating scorer reliability or assigning differential confidence to individual scorers. The calibration step removes differences in numerical scale, allowing each scorer to contribute equally according to its position within its own reference distribution. The set of subset-specific scorers is defined before calibration and held fixed during inference so that the weighting induced by the fusion rule remains consistent. Proposition~\ref{prop:fusion_consistency}, proved in Appendix~\ref{app:props_proofs}, shows that the fused supervision score ranks instances according to aggregate calibrated support across scorers. Together with Proposition~\ref{prop:ordinal_preservation}, this establishes that calibration preserves scorer-specific ordinal information while fusion combines that information into a single supervision score.

\section{Evaluation}
\label{sec:exp}
We evaluate the proposed framework through cross-domain validation and application-specific analyses. The cross-domain experiments evaluate the resulting supervision scores against held-out labels, while the application-specific experiments examine whether the framework exhibits the behavior characterized by Propositions~\ref{prop:ordinal_preservation} and~\ref{prop:fusion_consistency}. Given the framework's conceptual relationship to latent-variable measurement, we additionally compare it with traditional psychometric methods in Appendix~\ref{app:psychometrics}.

\subsection{Cross-Domain Validation}
\label{sec:exp:cross_domain}

We evaluate the performance correspondence of calibration-first fusion with held-out targets relative to several established and natural comparison baselines. 
Specifically, we carry out this assessment across three datasets from unrelated domains:
Ames Housing~\citep{decockAmesIowaAlternative2011}, a regression benchmark;
Breast Cancer Wisconsin~\citep{breast_cancer_wisconsin_diagnostic_17}, a medical classification benchmark; and
Wine Quality~\citep{cortez2009modeling}, a physicochemical benchmark. 
For each dataset, we compare against empirical \( z \)-scoring, unsupervised PCA, Borda rank averaging, uncalibrated averaging, and the best individual subset-specific scorer's value.

For each dataset, we defined four heuristic subset-specific scorers whose feature domains are disjoint.
Following Section~\ref{sec:framework:phases:calibration}, we constructed the synthetic ordinal reference space with \( k = 5 \) levels for each of the benchmark's features. Full scorer definitions and calibration ranges are provided in Appendix~\ref{app:cross_domain}.

Table~\ref{tab:cross_domain} summarizes the results, with 95\% confidence intervals for the \textit{primary metrics} constructed using \( B = 50{,}000 \) bootstrap resamples.
Pairwise method comparisons were conducted using paired bootstrap resampling, with the same resampled observations used for all methods within each replicate.
To maintain a family-wise confidence level of at least 95\% across three pairwise comparisons, we constructed Bonferroni-adjusted 98.33\% confidence intervals.

\begin{table*}[t]
\centering
\small
\begin{tabular}{llccc}
\toprule
\textbf{Dataset}
& \textbf{Method}
& \textbf{Primary [95\% CI]}
& \textbf{$\Delta=\kappa-\mathrm{method}$ [95\% paired CI]}
& \textbf{Bonf. [98.33\% CI]} \\
\midrule

Ames Housing
& Calibrated Fusion ($\kappa$)
& $0.890\,[0.876,\,0.903]$
& ---
& --- \\

& Empirical $z$-Score Avg.
& $0.908\,[0.895,\,0.919]$
& $-0.0172\,[-0.0225,\,-0.0122]$
& $[-0.0238,\,-0.0111]$ \\

& Borda Rank Averaging
& $0.905\,[0.893,\,0.916]$
& $-0.0148\,[-0.0214,\,-0.0085]$
& $[-0.0230,\,-0.0072]$ \\

& Unsupervised PCA (PC1)
& $\mathbf{0.909\,[0.896,\,0.919]}$
& $-0.0181\,[-0.0233,\,-0.0131]$
& $[-0.0245,\,-0.0120]$ \\

& Uncalibrated Average
& $0.859\,[0.840,\,0.877]$
& $+0.0311\,[0.0143,\,0.0485]$
& $[0.0109,\,0.0524]$ \\

& Best Single Scorer
& $0.8150$
& ---
& --- \\

\midrule

Breast Cancer
& Calibrated Fusion ($\kappa$)
& $0.984\,[0.976,\,0.991]$
& ---
& --- \\

& Empirical $z$-Score Avg.
& $0.983\,[0.974,\,0.990]$
& $+0.0017\,[0.0004,\,0.0030]$
& $[0.0002,\,0.0034]$ \\

& Borda Rank Averaging
& $0.984\,[0.975,\,0.992]$
& $+0.0002\,[-0.0034,\,0.0041]$
& $[-0.0041,\,0.0051]$ \\

& Unsupervised PCA (PC1)
& $\mathbf{0.987\,[0.978,\,0.993]}$
& $-0.0024\,[-0.0049,\,-0.0001]$
& $[-0.0055,\,0.0003]$ \\

& Uncalibrated Average
& $0.957\,[0.940,\,0.972]$
& $+0.0268\,[0.0140,\,0.0414]$
& $[0.0115,\,0.0449]$ \\

& Best Single Scorer
& $0.9705$
& ---
& --- \\

\midrule

Wine Quality
& Calibrated Fusion ($\kappa$)
& $\mathbf{0.556\,[0.520,\,0.590]}$
& ---
& --- \\

& Empirical $z$-Score Avg.
& $0.546\,[0.510,\,0.582]$
& $+0.0093\,[0.0012,\,0.0176]$
& $[-0.0005,\,0.0194]$ \\

& Borda Rank Averaging
& $0.549\,[0.513,\,0.584]$
& $+0.0069\,[-0.0050,\,0.0189]$
& $[-0.0077,\,0.0215]$ \\

& Unsupervised PCA (PC1)
& $0.546\,[0.509,\,0.581]$
& $+0.0097\,[-0.0061,\,0.0257]$
& $[-0.0094,\,0.0293]$ \\

& Uncalibrated Average
& $0.192\,[0.142,\,0.240]$
& $+0.3641\,[0.3173,\,0.4112]$
& $[0.3062,\,0.4221]$ \\

& Best Single Scorer
& $0.4757$
& ---
& --- \\

\bottomrule
\end{tabular}
\caption{
Cross-domain validation on the primary metric.
Point estimates are reported with percentile-bootstrap 95\% confidence intervals \((B=50{,}000)\).
Pairwise intervals report \( \Delta=\kappa-\mathrm{method} \), with positive values favoring calibrated fusion.
Bonferroni-adjusted 98.33\% confidence intervals control the family-wise error rate across the three datasets separately for each baseline.
The primary metric is Spearman's \( \rho \) for Ames Housing and Wine Quality and AUROC for Breast Cancer Wisconsin. Additional metrics are available in Appendix~\ref{app:cross_domain}.
}
\label{tab:cross_domain}
\end{table*}

Across all three domains, calibrated fusion improves over uncalibrated averaging on the primary metric. 
Moreover, Bonferroni-adjusted paired bootstrap confidence intervals exclude zero in each case, and achieves higher primary-metric point estimates than the best individual scorer.

Furthermore, our results indicate that for empirical \( z \)-scoring, Borda rank averaging, and unsupervised PCA, performance is domain dependent.
On Ames Housing, calibrated fusion is less than \( 0.02 \) lower than all three metrics, with all differences remaining distinguishable from zero after Bonferroni correction.
On Breast Cancer Wisconsin, calibrated fusion is higher than empirical \( z \)-score averaging after correction by \( 0.0017 \), while no adjusted difference is detected relative to Borda rank averaging or PCA.
On Wine Quality, no adjusted difference is detected between calibrated fusion and any of the three metrics.

In general, these results provide a clear indication that our methodology is a viable approach, given that no statistically significant differences were observed for the majority of the metrics. For the remaining metrics, the observed differences were relatively small.

\subsection{Empirical Illustration of Framework Properties}
\label{sec:exp:synth}

To exemplify our framework, we label 247 synthetic homes generated using the pipeline proposed by \citet{eshbaugh2026synthetichomesmultimodalgenerative}. In this setting, the latent concept is residential energy retrofit priority, and the synthetic ordinal reference space is defined by four calibration features: HVAC heating COP, HVAC cooling COP, roof R-value, and wall R-value. Because no ground truth labels are available in this synthetic dataset, we aim to approximate the latent concept as a source of supervision using our framework. Four subset-specific scorers are employed. Each home inspection note is evaluated by an LLM to produce separate HVAC and insulation retrofit-need scores. Two heuristic scorers operate on the structured calibration features: an HVAC scorer jointly evaluates heating and cooling COP, while an insulation scorer jointly evaluates roof and wall R-values.

For this dataset, the calibration reference space consists of all \( 625 \) combinations formed by selecting one of the five ordered levels for each of the four calibration features \( (5^4) \). These levels were selected using domain knowledge to span a broad range of residential efficiency conditions while preserving the expected ordinal relationship between each feature and retrofit need. The four scorer outputs are standardized and fused to produce concept-specific HVAC and insulation scores and an overall supervision score \( \kappa \), with larger values indicating greater retrofit need. Additional implementation details, including the exact calibration values and their interpretation, are provided in Appendix~\ref{app:synthetic_homes}.


We perform a sensitivity analysis to verify that each scorer responds monotonically to its designated input, and that calibration preserves this monotonicity. The inputs to one subset-specific scorer are varied across five ordered efficiency levels while the inputs to the remaining scorers are held fixed at the neutral third level. Level 1 represents the least efficient condition, while Level 5 represents the most efficient condition. Table~\ref{sec:exp:abl:tbl:results} reports the averages of standardized scorer outputs in which negative values indicate retrofit need below the calibration-space mean, while positive values indicate retrofit need above the mean. Across all four experiments, the subset-specific supervision scores did not increase as efficiency improved, indicating that increasing efficiency did not produce an increase in estimated retrofit need. The calibrated scorer outputs also preserved every ordering expressed by their corresponding raw outputs. Across the complete calibration space, no ordinal preservation failures occurred in the \( 195{,}000 \) pairwise comparisons evaluated for each scorer\footnote{The equal scores observed between HVAC levels 4 and 5 and between insulation levels 1 and 2 reflect ties in the corresponding LLM outputs rather than ordering reversals introduced by calibration.}. Hence, the ordinal preservation guarantee in Proposition~\ref{prop:ordinal_preservation} is satisfied empirically for all four subset-specific scorers over the complete calibration space.


\begin{table*}[t]
    \centering
    \begin{tabular}{lrrrrr}
        \toprule
        \textbf{Varied Scorer}
        & \textbf{Level 1}
        & \textbf{Level 2}
        & \textbf{Level 3}
        & \textbf{Level 4}
        & \textbf{Level 5} \\
        \midrule
        Text / HVAC
        & 1.10
        & 0.43
        & 0.26
        & -0.57
        & -0.57 \\

        Structured / HVAC
        & 1.26
        & 0.76
        & 0.26
        & -0.24
        & -0.74 \\

        Text / Insulation
        & 0.94
        & 0.94
        & -0.09
        & -0.24
        & -0.53 \\

        Structured / Insulation
        & 0.91
        & 0.41
        & -0.09
        & -0.59
        & -1.09 \\
        \bottomrule
    \end{tabular}
    \caption{Concept-specific supervision scores under scorer-level sensitivity analysis. Larger values indicate greater retrofit need.}
    \label{sec:exp:abl:tbl:results}
\end{table*}


\begin{table*}[t]
    \centering
    \begin{tabular}{llrr}
    \toprule
    \textbf{Text Evidence}
    & \textbf{Structured Evidence}
    & \textbf{HVAC}
    & \textbf{Insulation} \\
    \midrule
    Level 1 & Level 1 & 2.10 & 1.94 \\
    Level 1 & Level 5 & 0.10 & -0.06 \\
    Level 5 & Level 1 & 0.43 & 0.47 \\
    Level 5 & Level 5 & -1.57 & -1.53 \\
    \bottomrule
\end{tabular}
    \caption{Concept-specific supervision scores. Level 1 is least efficient and Level 5 is most efficient.}
    \label{tab:combined_variation_extremes}
\end{table*}

Additionally, we evaluate the consistency of the outputs of subset-specific scorers in order to validate that their results reflect the quality of their inputs. This process consists of comparing four extreme combinations of the inputs in both formats: text-derived and structured evidence. In Table~\ref{tab:combined_variation_extremes}, we observe that jointly inefficient evidence produced the highest retrofit-priority signal for both concepts, while jointly efficient evidence produced the lowest. The cases with combined conditions fell between these endpoints and were asymmetric given that the calibrated text-derived and structured contributions differed in magnitude. These results are consistent with Proposition~\ref{prop:fusion_consistency}.


\subsection{Calibration Resolution Sensitivity Analysis}
\label{sec:exp:k_sensitivity}

The parameter \( k \) determines both the resolution of the synthetic calibration space and its computational cost. We evaluate whether dense calibration grids are necessary by measuring the tradeoff between calibration resolution and agreement with a high-resolution reference space. Specifically, we evaluate \( k \in \{2,3,5,7,9,11,13\} \) against a \( k=15 \) reference space. The \(k=15\) space was constructed by defining 15 ordered values for each calibration feature through interpolation between the predefined physical anchor values described in Appendix~\ref{app:calibration_space}. Lower-resolution spaces were constructed as nested subsets of this reference grid by selecting approximately equally spaced ordinal levels. This procedure preserves the endpoints and maintains comparable coverage of the feature ranges while reducing the number of calibration points evaluated. The selected levels and corresponding physical values are provided in Appendix~\ref{app:calibration_space}.

As resolution increases, agreement with the \(k=15\) reference improves, while the number of calibration points grows exponentially, from \(16\) points at \(k=2\) to \(50{,}625\) points at \(k=15\). Despite this growth, the induced rankings stabilize at substantially lower resolutions. A resolution of \(k=7\) recovers a Spearman correlation of \(0.9939\) with the \(k=15\) reference, while \(k=9\) achieves \(0.9987\) correlation and \(96.7\%\) top-\(10\%\) overlap using only \(12.9\%\) of the reference calibration points. Figure~\ref{fig:spearman_cost_tradeoff} illustrates the resulting tradeoff between calibration cost and agreement. Increasing from \(k=9\) to \(k=15\) improves Spearman correlation by only \(0.0013\), while increasing EnergyPlus computation time from approximately \(2.6\) hours to nearly \(20\) hours. Thus, for the Synthetic Homes case study, \(k=9\) provides a near-reference agreement while avoiding the computational expense of the full reference grid. These results suggest that intermediate calibration resolutions may provide a favorable balance between computational cost and agreement with higher-resolution references.

\begin{figure}[t]
    \centering
    \includegraphics[width=0.80\linewidth]{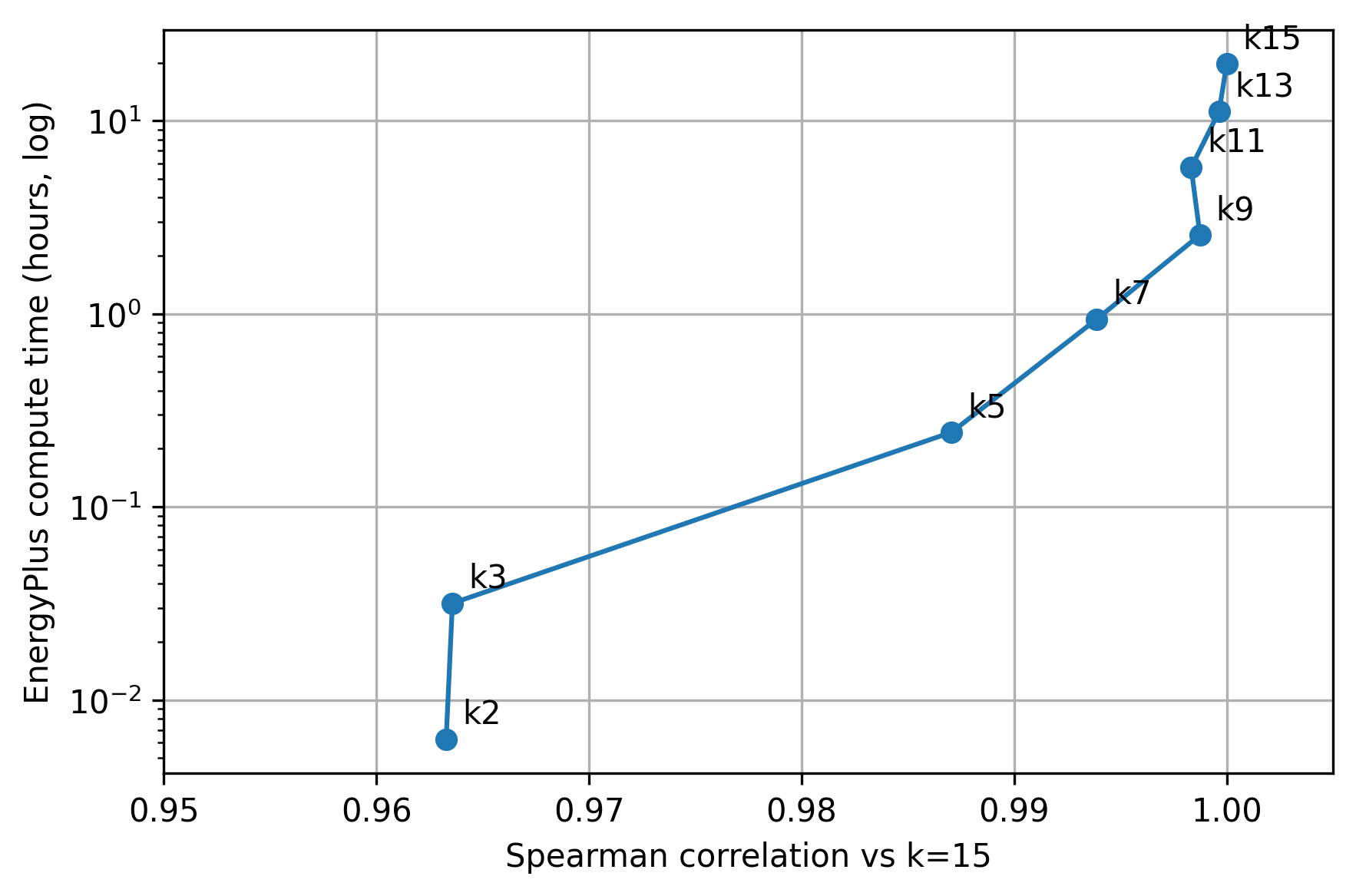}
    \caption{Spearman \( \rho \) against \( k = 15\) vs EnergyPlus compute time for all tested values of \( k \).}
    \label{fig:spearman_cost_tradeoff}
\end{figure}

\section{Conclusion}
\label{sec:conc}

In this work, we showed that heterogeneous evidence sources without ground-truth labels or a shared annotation space can be combined into a common supervision signal using our calibration-first framework.
In a residential energy retrofit case study, the implementation exhibited the predicted monotonic behavior and remained stable across calibration resolutions.
Additionally, we examined the tradeoff between computational cost and reference space resolution, finding that lower-resolution grids closely approximate higher-resolution results while substantially reducing computation.

Across three benchmark datasets, calibrated fusion consistently outperformed uncalibrated averaging on the primary metric, with differences remaining distinguishable from zero after Bonferroni correction, and achieved higher primary-metric point estimates than the best individual scorer.
Relative to empirical \( z \)-scoring, Borda rank averaging, and unsupervised PCA, performance was domain-dependent: calibrated fusion was within 0.02 of all three baselines on Ames Housing, with each difference remaining distinguishable from zero after correction.
On Breast Cancer Wisconsin, one adjusted difference was detected, and on Wine Quality, no adjusted differences were detected.
These findings suggest that our framework performs competitively with empirical \( z \)-scoring, Borda rank averaging, and unsupervised PCA.

Unlike these metrics, our framework estimates scorer-specific calibration statistics from an externally defined synthetic ordinal reference space rather than from the observed sample.
As a consequence, the proposed calibration procedure supports score comparability across cohorts, stability under changes in sample composition, transferability across datasets, and clearer separation between scorer behavior and the distribution of the observed sample.

Our framework, similarly to other weak-supervision approaches, relies on domain knowledge, which is encoded through the selection and ordering of calibration features.
Because these choices determine the structure of the calibration space, future work should examine how sensitive the resulting supervision is the assumptions made during its construction.
Additional directions include incorporating scorer reliability into the fusion procedure and developing principled criteria for selecting calibration resolutions.

Ultimately, supervision without ground truth requires more than combining heterogeneous signals; it requires making them comparable. Calibration-first fusion provides this missing step by aligning scorer outputs to a common scale before fusion.

\subsection*{Acknowledgments}
The authors extend gratitude to Lafayette College for supporting this work by funding it through the EXCEL Scholars program and providing computational resources with the Firebird high performance computing cluster. 
This work also used Jetstream2 at Indiana University through allocation PHY240196 from the Advanced Cyberinfrastructure Coordination Ecosystem: Services \& Support (ACCESS) program~\citep{10.1145/3569951.3597559}, which is supported by U.S. National Science Foundation grants \#2138259, \#2138286, \#2138307, \#2137603, and \#2138296.

\printbibliography

@article{eshbaugh2026synthetichomesmultimodalgenerative,
title = {Synthetic homes: A multimodal generative AI pipeline for residential building data generation under data scarcity},
journal = {Machine Learning with Applications},
volume = {25},
pages = {100959},
year = {2026},
issn = {2666-8270},
doi = {https://doi.org/10.1016/j.mlwa.2026.100959},
author = {Jackson Eshbaugh and Chetan Tiwari and Jorge Silveyra},
}

@article{Ratner2017,
  title = {Snorkel: rapid training data creation with weak supervision},
  volume = {11},
  ISSN = {2150-8097},
  url = {http://dx.doi.org/10.14778/3157794.3157797},
  DOI = {10.14778/3157794.3157797},
  number = {3},
  journal = {Proceedings of the VLDB Endowment},
  publisher = {Association for Computing Machinery (ACM)},
  author = {Ratner,  Alexander and Bach,  Stephen H. and Ehrenberg,  Henry and Fries,  Jason and Wu,  Sen and Ré,  Christopher},
  year = {2017},
  month = nov,
  pages = {269–282}
}

@article{Baltrusaitis2019,
  title = {Multimodal Machine Learning: A Survey and Taxonomy},
  volume = {41},
  ISSN = {1939-3539},
  url = {http://dx.doi.org/10.1109/TPAMI.2018.2798607},
  DOI = {10.1109/tpami.2018.2798607},
  number = {2},
  journal = {IEEE Transactions on Pattern Analysis and Machine Intelligence},
  publisher = {Institute of Electrical and Electronics Engineers (IEEE)},
  author = {Baltrusaitis,  Tadas and Ahuja,  Chaitanya and Morency,  Louis-Philippe},
  year = {2019},
  month = feb,
  pages = {423–443}
}

@article{Liang2024,
  title = {Foundations \& Trends in Multimodal Machine Learning: Principles,  Challenges,  and Open Questions},
  volume = {56},
  ISSN = {1557-7341},
  url = {http://dx.doi.org/10.1145/3656580},
  DOI = {10.1145/3656580},
  number = {10},
  journal = {ACM Computing Surveys},
  publisher = {Association for Computing Machinery (ACM)},
  author = {Liang,  Paul Pu and Zadeh,  Amir and Morency,  Louis-Philippe},
  year = {2024},
  month = jun,
  pages = {1–42}
}

@inproceedings{ngiam2011multimodal,
  title     = {Multimodal Deep Learning},
  author    = {Ngiam, Jiquan and Khosla, Aditya and Kim, Mingyu and Nam, Juhan and Lee, Honglak and Ng, Andrew Y.},
  booktitle = {Proceedings of the 28th International Conference on Machine Learning (ICML)},
  pages     = {689--696},
  year      = {2011},
  address   = {Bellevue, Washington, USA},
  publisher = {Omnipress}
}

@inproceedings{boecking2021interactive,
  title = {Interactive Weak Supervision: {{Learning}} Useful Heuristics for Data Labeling},
  booktitle = {International Conference on Learning Representations},
  author = {Boecking, Benedikt and Neiswanger, Willie and Xing, Eric and Dubrawski, Artur},
  year = 2021
}

@inproceedings{cachay2021endtoend,
  title = {End-to-End Weak Supervision},
  booktitle = {Advances in {{Neural Information Processing Systems}}},
  author = {Cachay, Salva R{\"u}hling and Boecking, Benedikt and Dubrawski, Artur},
  editor = {Beygelzimer, A. and Dauphin, Y. and Liang, P. and Vaughan, J. Wortman},
  year = 2021
}

@inproceedings{pmlr-v151-yu22c,
  title = {Learning from Multiple Noisy Partial Labelers},
  booktitle = {Proceedings of the 25th International Conference on Artificial Intelligence and Statistics},
  author = {Yu, Peilin and Ding, Tiffany and Bach, Stephen H.},
  editor = {{Camps-Valls}, Gustau and Ruiz, Francisco J. R. and Valera, Isabel},
  year = 2022,
  month = mar,
  series = {Proceedings of Machine Learning Research},
  volume = {151},
  pages = {11072--11095},
  publisher = {PMLR}
}

@misc{zhangSurveyProgrammaticWeak2022,
  title = {A {{Survey}} on {{Programmatic Weak Supervision}}},
  author = {Zhang, Jieyu and Hsieh, Cheng-Yu and Yu, Yue and Zhang, Chao and Ratner, Alexander},
  year = 2022,
  month = feb,
  number = {arXiv:2202.05433},
  eprint = {2202.05433},
  primaryclass = {cs.LG},
  publisher = {arXiv},
  doi = {10.48550/arXiv.2202.05433},
  urldate = {2026-05-27},
  archiveprefix = {arXiv}
}

@article{clarkConstructingValidityNew2019,
  title = {Constructing Validity: {{New}} Developments in Creating Objective Measuring Instruments.},
  shorttitle = {Constructing Validity},
  author = {Clark, Lee Anna and Watson, David},
  year = 2019,
  month = dec,
  journal = {Psychological Assessment},
  volume = {31},
  number = {12},
  pages = {1412--1427},
  issn = {1939-134X, 1040-3590},
  doi = {10.1037/pas0000626},
  urldate = {2026-06-15},
  copyright = {http://www.apa.org/pubs/journals/resources/open-access.aspx},
  langid = {english}
}

@article{cronbachConstructValidityPsychological1955,
  title = {Construct Validity in Psychological Tests.},
  author = {Cronbach, Lee J. and Meehl, Paul E.},
  year = 1955,
  month = jul,
  journal = {Psychological Bulletin},
  volume = {52},
  number = {4},
  pages = {281--302},
  issn = {1939-1455, 0033-2909},
  doi = {10.1037/h0040957},
  urldate = {2026-06-15},
  langid = {english}
}

@incollection{furrPsychometrics2008,
  title = {Item {{Response Theory}} and {{Rasch Models}}},
  booktitle = {Psychometrics: {{An Introduction}}},
  author = {Furr, R. Michael and Bacharach, Verne R.},
  year = 2008,
  pages = {314--334},
  publisher = {Sage},
  address = {Thousand Oaks, California},
  urldate = {2026-06-15},
  isbn = {978-1-4129-2760-4}
}

@book{raschProbabilisticModelsIntelligence1980,
  title = {Probabilistic Models for Some Intelligence and Attainment Tests},
  author = {Rasch, G.},
  year = 1980,
  edition = {Expanded ed},
  publisher = {University of Chicago Press},
  address = {Chicago},
  isbn = {978-0-226-70553-8 978-0-226-70554-5},
  lccn = {BF431 .R338 1980}
}

@inproceedings{ICLR2025_e91fb65c,
 author = {Agrawal, Vishwajeet and Pukdee, Rattana and Balcan, Nina and Ravikumar, Pradeep K},
 booktitle = {International Conference on Learning Representations},
 editor = {Y. Yue and A. Garg and N. Peng and F. Sha and R. Yu},
 pages = {93153--93193},
 title = {Learning from weak labelers as constraints},
 url = {https://proceedings.iclr.cc/paper_files/paper/2025/file/e91fb65c6324a984ea9ef39a5b84af04-Paper-Conference.pdf},
 volume = {2025},
 year = {2025}
}

@misc{bacellarCalibratedFusionHeterogeneous2026,
  author = {Bacellar, Andre},
  title = {Calibrated Fusion for Heterogeneous Graph-Vector Retrieval in Multi-Hop QA},
  year = {2026},
  howpublished = {arXiv preprint arXiv:2603.28886},
  doi = {10.48550/arXiv.2603.28886}
}

@article{yazbeckLatentSupervisionMethod2026,
  title = {Latent {{Supervision}}: {{A Method}} for {{Improved Performance}} and {{Calibration}} of {{Machine Learning Classification Models}} in {{Ophthalmology}}},
  shorttitle = {Latent {{Supervision}}},
  author = {Yazbeck, Hady and Assaf, Jad and Lietman, Tom M. and Keenan, Jeremy D. and Jackson, John and {Kalpathy-Cramer}, Jayashree and Campbell, J. Peter and Song, Xubo and Redd, Travis K.},
  year = 2026,
  month = may,
  journal = {Ophthalmology Science},
  volume = {6},
  number = {5},
  pages = {101151},
  issn = {26669145},
  doi = {10.1016/j.xops.2026.101151},
  urldate = {2026-06-15},
  langid = {english}
}

@book{embretsonItemResponseTheory2025,
  title = {Item {{Response Theory}}: {{Foundations}} for {{Psychologists}} and {{Social Scientists}}},
  shorttitle = {Item {{Response Theory}}},
  author = {Embretson, Susan E. and Reise, Steven P.},
  year = 2025,
  month = oct,
  edition = {2},
  publisher = {Routledge},
  address = {New York},
  doi = {10.4324/9781315726557},
  urldate = {2026-06-15},
  isbn = {978-1-315-72655-7},
  langid = {english}
}

@misc{breast_cancer_wisconsin_diagnostic_17,
  author = {Wolberg, William and Mangasarian, Olvi and Street, Nick and Street, W.},
  title = {Breast Cancer Wisconsin (Diagnostic)},
  year = {1993},
  howpublished = {UCI Machine Learning Repository},
  note = {DOI: \url{https://doi.org/10.24432/C5DW2B}}
}

@article{decockAmesIowaAlternative2011,
  title = {Ames, {{Iowa}}: {{Alternative}} to the {{Boston Housing Data}} as an {{End}} of {{Semester Regression Project}}},
  shorttitle = {Ames, {{Iowa}}},
  author = {De Cock, Dean},
  year = 2011,
  month = nov,
  journal = {Journal of Statistics Education},
  volume = {19},
  number = {3},
  pages = {8},
  issn = {1069-1898},
  doi = {10.1080/10691898.2011.11889627},
  urldate = {2026-06-15},
  langid = {english}
}

@misc{OpenPsychometrics,
  author = {{OpenPsychometrics}},
  title = {OpenPsychometrics},
  year = {2019},
  howpublished = {\url{https://openpsychometrics.org}},
  note = {Accessed 2026-06-15}
}

@article{goldbergDevelopmentMarkersBigFive1992,
  title = {The Development of Markers for the {{Big-Five}} Factor Structure.},
  author = {Goldberg, Lewis R.},
  year = 1992,
  month = mar,
  journal = {Psychological Assessment},
  volume = {4},
  number = {1},
  pages = {26--42},
  issn = {1939-134X, 1040-3590},
  doi = {10.1037/1040-3590.4.1.26},
  urldate = {2026-06-15},
  langid = {english}
}

@article{leePsychometricPropertiesHEXACO2004,
  title = {Psychometric {{Properties}} of the {{HEXACO Personality Inventory}}},
  author = {Lee, Kibeom and Ashton, Michael C.},
  year = 2004,
  month = apr,
  journal = {Multivariate Behavioral Research},
  volume = {39},
  number = {2},
  pages = {329--358},
  issn = {0027-3171, 1532-7906},
  doi = {10.1207/s15327906mbr3902_8},
  urldate = {2026-06-15},
  langid = {english}
}

@article{KirkegaardJensen2025MGKT,
  author = {Kirkegaard, Emil O. W. and Jensen, Sebastian},
  title = {Psychometric Analysis of the Multifactor General Knowledge Test},
  journal = {OpenPsych},
  year = {2025},
  doi = {10.26775/OP.2025.07.28}
}

@misc{OpenPsychometricsMGKT,
  author = {{OpenPsychometrics}},
  title = {Multifactor General Knowledge Test (MGKT)},
  year = {2019},
  howpublished = {\url{https://openpsychometrics.org/tests/MGKT/}},
  note = {Accessed 2026-06-15}
}

@article{cortez2009modeling,
  title={Modeling wine preferences by data mining from physicochemical properties},
  author={Cortez, Paulo and Cerdeira, Ant{\'o}nio and Almeida, Fernando and Matos, Telmo and Reis, Jos{\'e}},
  journal={Decision Support Systems},
  volume={47},
  number={4},
  pages={547--553},
  year={2009},
  publisher={Elsevier}
}

@inproceedings{10.1145/3569951.3597559,
author = {Boerner, Timothy J. and Deems, Stephen and Furlani, Thomas R. and Knuth, Shelley L. and Towns, John},
title = {ACCESS: Advancing Innovation: NSF’s Advanced Cyberinfrastructure Coordination Ecosystem: Services \& Support},
year = {2023},
isbn = {9781450399852},
publisher = {Association for Computing Machinery},
address = {New York, NY, USA},
url = {https://doi.org/10.1145/3569951.3597559},
doi = {10.1145/3569951.3597559},
booktitle = {Practice and Experience in Advanced Research Computing 2023: Computing for the Common Good},
pages = {173–176},
numpages = {4},
location = {Portland, OR, USA},
series = {PEARC '23}
}

\appendix

\section{Propositions \& Proofs}
\label{app:props_proofs}
\begin{proposition}[Ordinal Preservation]
\label{prop:ordinal_preservation}
Let \(h\) denote the latent concept, and let \(s_j\) be a subset-specific scorer such that
\[
h(a)<h(b) \implies s_j(a)<s_j(b)
\]
for any instances \(a,b\). Let \(\hat{s}_j\) denote the standardized scorer obtained using calibration statistics with \(\sigma_j>0\). Then
\[
h(a)<h(b) \implies \hat{s}_j(a)<\hat{s}_j(b).
\]
Thus, calibration preserves any ordinal structure in \(s_j\) that is consistent with the latent concept.
\end{proposition}
\begin{proof}
By assumption, \(h(a)<h(b)\) implies \(s_j(a)<s_j(b)\). The calibrated score is
\[
\hat{s}_j(x)=\frac{s_j(x)-\mu_j}{\sigma_j}.
\]
Since \(\sigma_j>0\), this transformation is affine with positive slope \(1/\sigma_j\). Therefore
\[
s_j(a)<s_j(b)
\implies
\frac{s_j(a)-\mu_j}{\sigma_j}
<
\frac{s_j(b)-\mu_j}{\sigma_j},
\]
so \(\hat{s}_j(a)<\hat{s}_j(b)\). Hence \(h(a)<h(b)\implies \hat{s}_j(a)<\hat{s}_j(b)\).
\end{proof}

This result means that calibration neither creates nor destroys ordinal information. If subset-specific scorers are monotone observations of the latent concept, then calibrated scores preserve the latent ordering while expressing each score relative to the same reference distribution.

\begin{proposition}[Fusion Consistency]
\label{prop:fusion_consistency}

Let

\[
\kappa(x)=\frac1m\sum_{j=1}^{m}\hat s_j(x).
\]

For any two instances \(a\) and \(b\), define \( \Delta_j=\hat s_j(b)-\hat s_j(a). \) Then \( \kappa(b)>\kappa(a) \) if and only if

\[
\sum_{j:\Delta_j>0}\Delta_j
>
\sum_{j:\Delta_j<0}|\Delta_j|.
\]

Thus, the fused supervision score ranks instances according to aggregate calibrated support across subset-specific scorers.

\end{proposition}
\begin{proof}
Let \( \Delta_j = \hat{s}_j(b) - \hat{s}_j(a) \) represent the calibrated difference for subset-specific scorer \( j \). Let \( A = \{ j \mid \Delta_j > 0 \} \) be the set of scorers that favor \( b \), and let \( B = \{ j \mid \Delta_j < 0 \} \) be the set of scorers that favor \( a \).

Define

\[
\Delta_\kappa
=
\kappa(b)-\kappa(a)
=
\frac{\sum_{j=1}^{m}\Delta_j}{m}.
\]

Since \(m>0\),

\[
\Delta_\kappa>0
\iff
\sum_{j=1}^{m}\Delta_j>0.
\]

Partitioning the sum into positive and negative contributions yields

\[
\sum_{j=1}^{m}\Delta_j
=
\sum_{j\in A}\Delta_j
-
\sum_{j\in B}|\Delta_j|.
\]

Therefore,

\[
\Delta_\kappa>0
\iff
\sum_{j\in A}\Delta_j
>
\sum_{j\in B}|\Delta_j|.
\]

Hence fusion ranks \(b>a\) if and only if the aggregate calibrated support for \(b\) exceeds the aggregate calibrated support for \(a\).
\end{proof}

\section{Calibration Space Construction \& Resolution Analysis}
\label{app:calibration_space}
This appendix provides additional details regarding the construction of the synthetic calibration space and the selection of calibration resolution. The calibration space defines the shared reference used to align subset-specific scorers prior to fusion. Because the number of calibration points grows exponentially with the number of calibration features and the number of ordinal levels selected for each feature, we also evaluate the computational tradeoff associated with increasing calibration resolution.

\subsection{Selection of Calibration Levels}
\label{app:calibration_levels}

Each calibration feature is represented using an ordinal scale containing \(k\) ordered levels. Calibration levels are selected to span the feasible range of each feature while preserving the ordinal relationships between levels. For the Synthetic Homes case study, the calibration features are HVAC heating COP, HVAC cooling COP, roof R-value, and wall R-value. Each feature range is discretized independently, with lower values representing less favorable conditions and higher values representing more favorable conditions.

For the main experiments, calibration levels are selected as approximately evenly spaced points across the feasible range of each calibration feature. This procedure provides a uniform ordinal reference space while avoiding arbitrary concentration of calibration points in specific regions of the feature space.

For the resolution sensitivity analysis, the highest-resolution calibration space (\(k=15\)) serves as the reference space. Lower-resolution calibration spaces are constructed as nested subsets of this reference grid by selecting approximately equally spaced ordinal levels. This preserves the endpoints and maintains similar coverage of the underlying feature ranges while reducing the number of calibration points evaluated.

Table~\ref{tab:calibration_level_selection} reports the selected ordinal levels for each evaluated value of \(k\). The corresponding physical parameter values used to construct the calibration space are provided in Table~\ref{tab:calibration_parameter_values}.

\begin{table*}[t]
\centering
\begin{tabular}{cc}
\toprule
Resolution & Selected levels \\
\midrule
$k=2$ & 1, 15 \\
$k=3$ & 1, 8, 15 \\
$k=5$ & 1, 5, 8, 11, 15 \\
$k=7$ & 1, 3, 6, 8, 10, 13, 15 \\
$k=9$ & 1, 3, 5, 6, 8, 10, 11, 13, 15 \\
$k=11$ & 1, 2, 4, 5, 7, 8, 9, 11, 12, 14, 15 \\
$k=13$ & 1, 2, 3, 5, 6, 7, 8, 9, 10, 11, 13, 14, 15 \\
$k=15$ & 1, 2, 3, 4, 5, 6, 7, 8, 9, 10, 11, 12, 13, 14, 15 \\

\bottomrule
\end{tabular}
\caption{Selected ordinal levels from the $k=15$ master calibration grid for each calibration resolution.}
\label{tab:calibration_level_selection}
\end{table*}

\begin{table*}[t]
\centering
\begin{tabular}{ccccc}
\toprule
Level & Wall $R$ & Roof $R$ & Heating COP & Cooling COP \\
\midrule
1 & 4 & 10 & 0.7 & 1 \\
2 & 5 & 13.333 & 0.7333 & 1.3333 \\
3 & 6 & 16.667 & 0.7667 & 1.6667 \\
4 & 7 & 20 & 0.8 & 2 \\
5 & 8.5 & 22.5 & 0.825 & 2.25 \\
6 & 10 & 25 & 0.85 & 2.5 \\
7 & 11.5 & 27.5 & 0.875 & 2.75 \\
8 & 13 & 30 & 0.9 & 3 \\
9 & 14.75 & 32.5 & 0.9125 & 3.125 \\
10 & 16.5 & 35 & 0.925 & 3.25 \\
11 & 18.25 & 37.5 & 0.9375 & 3.375 \\
12 & 20 & 40 & 0.95 & 3.5 \\
13 & 23.333 & 43.333 & 0.9667 & 3.6667 \\
14 & 26.667 & 46.667 & 0.9833 & 3.8333 \\
15 & 30 & 50 & 1 & 4 \\

\bottomrule
\end{tabular}
\caption{Physical values associated with the $k=15$ calibration levels.}
\label{tab:calibration_parameter_values}
\end{table*}

\subsection{Calibration Resolution and Computational Cost}
\label{app:calibration_cost}

The number of calibration points increases exponentially with calibration resolution. Because the calibration space varies four features independently, the total number of calibration points is \( |S| = k^4, \) where \(k\) denotes the number of ordinal levels assigned to each calibration feature. Consequently, increasing calibration resolution provides finer coverage of the reference space but requires additional scorer evaluations and, in the Synthetic Homes case study, additional EnergyPlus simulations.

Figure~\ref{fig:calibration_cost_scaling} illustrates the growth in calibration points as a function of \(k\). Although increasing resolution improves agreement with the highest-resolution reference space, the computational cost grows rapidly.

\begin{figure}[t]
    \centering
    \includegraphics[width=0.6\linewidth]{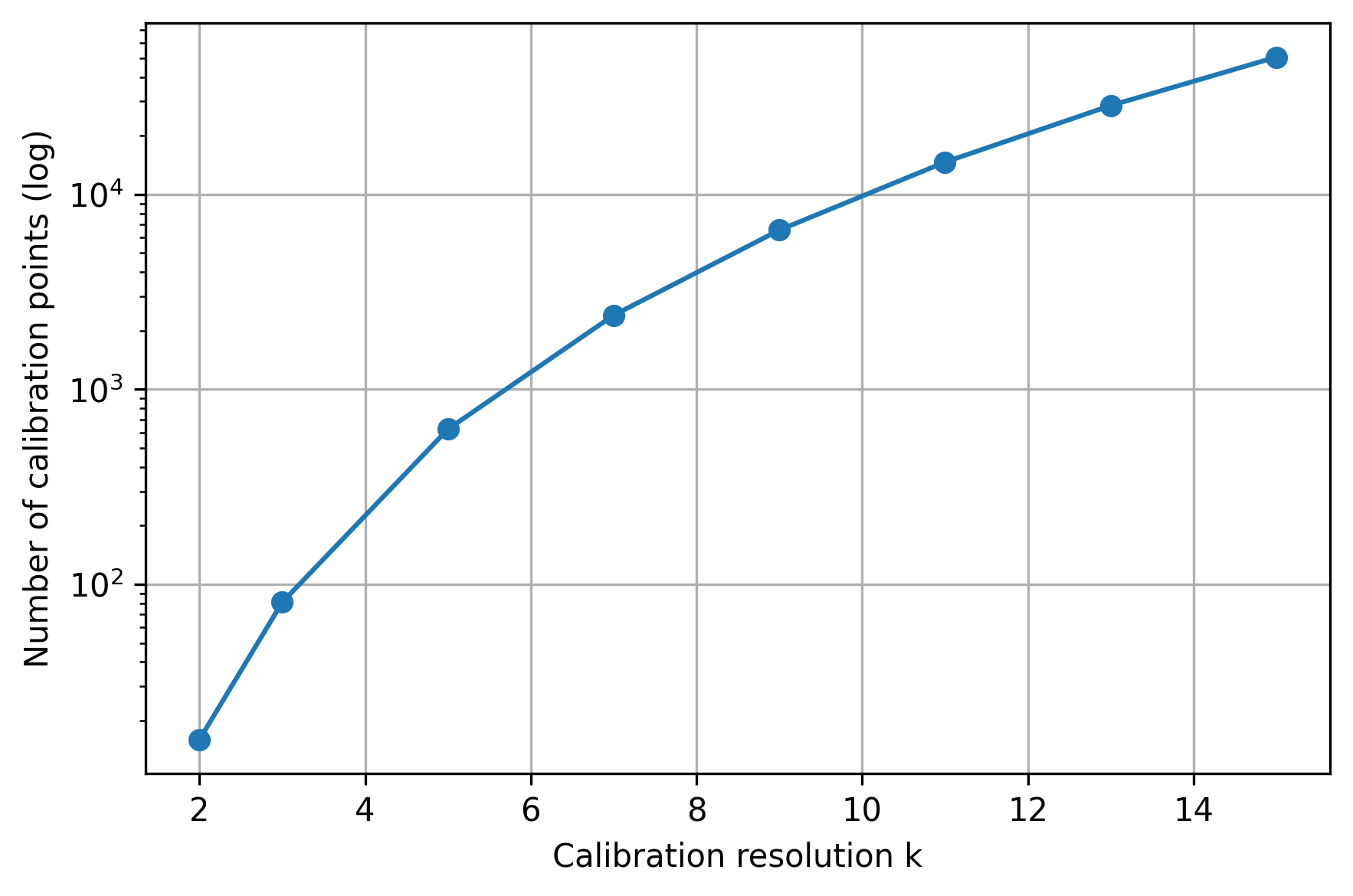}
    \caption{Growth in the number of calibration points with calibration resolution \(k\). The calibration space grows as \(k^4\) because four calibration features are varied independently.}
    \label{fig:calibration_cost_scaling}
\end{figure}

\subsection{Resolution Agreement and Diminishing Returns}
\label{app:resolution_agreement}

To evaluate the effect of calibration resolution on the resulting supervision signal, each lower-resolution calibration space is compared against the \(k=15\) reference space. Agreement is evaluated using rank correlation and differences between the resulting supervision scores.

Table~\ref{tab:k_sensitivity} reports the full results across the evaluated calibration resolutions, including calibration-space size, computation time, Spearman correlation, and top-10\% overlap with the \(k=15\) reference.

\begin{table*}[t]
\centering
\begin{tabular}{lcccc}
\toprule
\textbf{Resolution}
& \textbf{Points}
& \textbf{Runtime}
& \textbf{Spearman \(\rho\)}
& \textbf{Top-10\% Overlap} \\
\midrule
\(k=5\) & 625 & 0.24h & 0.9871 & 88.56 \% \\
\(k=7\) & 2401 & 0.94h & 0.9939 & 92.42 \% \\
\(k=9\) & 6561 & 2.56h & 0.9987 & 96.66 \% \\
\(k=15\) & 50625 & 19.76h & 1.0000 & 100.00 \% \\
\bottomrule
\end{tabular}
\caption{Sensitivity results from selected \( k \)-values.}
\label{tab:k_sensitivity}
\end{table*}

Figure~\ref{fig:spearman_resolution} shows the relationship between calibration resolution and Spearman correlation with the \(k=15\) reference. Agreement improves rapidly at lower resolutions before approaching saturation at higher resolutions.

\begin{figure}[t]
    \centering
    \includegraphics[width=0.6\linewidth]{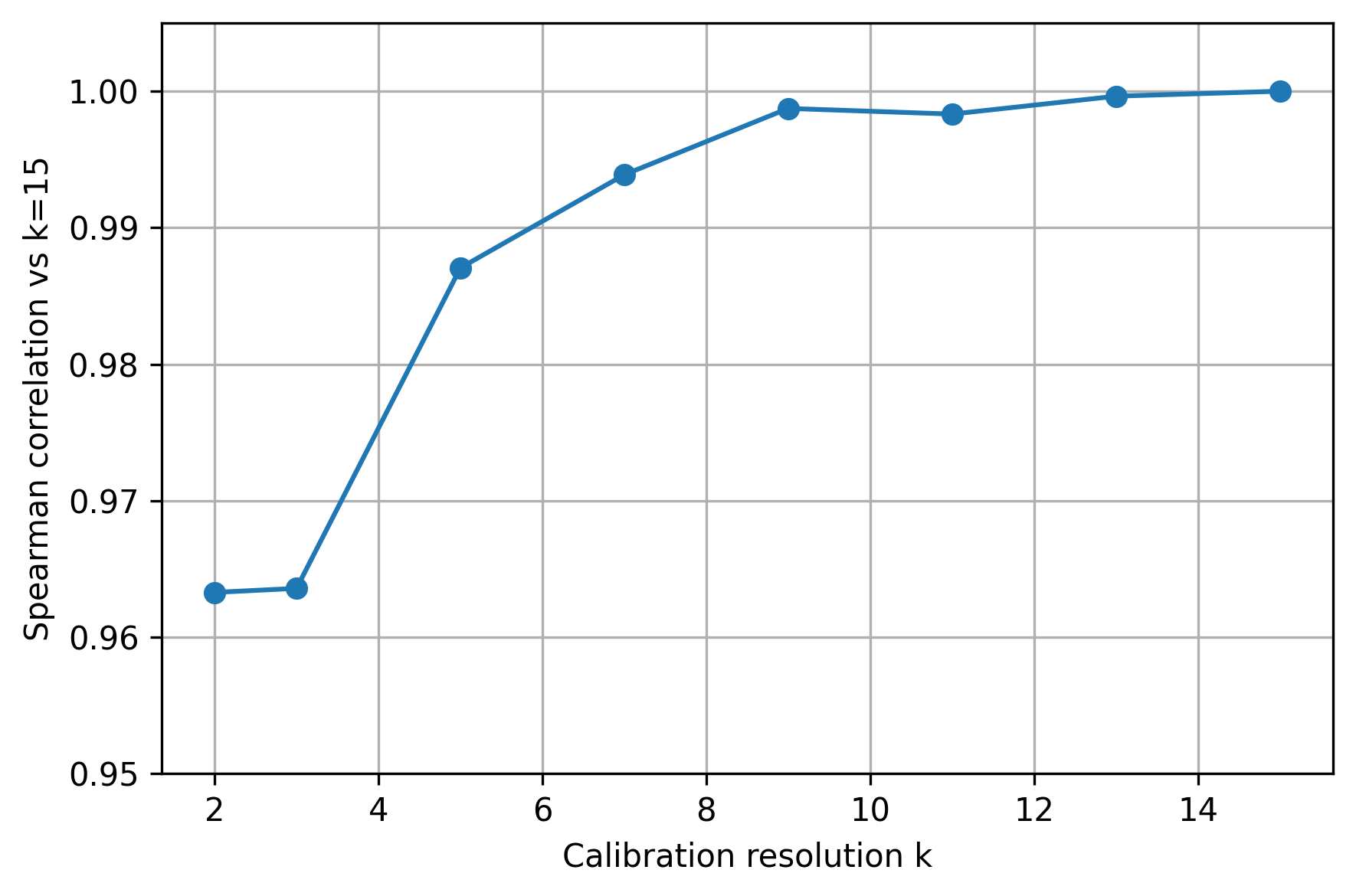}
    \caption{Spearman correlation between supervision scores generated at each calibration resolution and the \(k=15\) reference space.}
    \label{fig:spearman_resolution}
\end{figure}

While increasing \(k\) improves agreement, these improvements require increasingly larger calibration spaces. Figure~\ref{fig:spearman_cost_tradeoff} illustrates the tradeoff between agreement with the reference space and EnergyPlus computation time. Increasing the resolution beyond moderate values provides only marginal improvements in rank preservation while requiring substantially greater computational expense.

\subsection{Implications for Calibration Resolution Selection}
\label{app:resolution_selection}

The results demonstrate that increasing calibration resolution improves agreement with the highest-resolution reference space, but the marginal benefit decreases as the number of calibration points increases. Therefore, the appropriate value of \(k\) depends on the computational resources available and the precision required for a given application.

Future work may investigate automated approaches for selecting calibration resolution, analogous to elbow-based methods used in clustering. Such approaches could identify a point at which additional calibration points provide insufficient improvement relative to their computational cost.

\section{Cross-Domain Validation Details}
\label{app:cross_domain}
This appendix provides the scorer definitions and synthetic reference spaces
used for the cross-domain validation experiments in
Section~\ref{sec:exp:cross_domain}. In all three datasets, the target variable
was removed before supervision construction and used only for post hoc
evaluation. The feature groups were selected a priori based on semantic
correspondence among the dataset variables and their expected monotonic
relationship with the latent concept. Calibration endpoints were manually
specified to span low-to-high values for each feature. Neither the feature
groupings nor the endpoint values were selected using the held-out target
labels.

\subsection{Ames Housing}

The latent concept is residential property value. Four heuristic scorers
represent complementary aspects of the concept:
\begin{align}
s_{\mathrm{scale}}(x)
&=
\frac{
x_{\mathrm{GrLivArea}}
+
x_{\mathrm{TotalBsmtSF}}
}{2},
\\
s_{\mathrm{quality}}(x)
&=
\frac{
x_{\mathrm{OverallQual}}
+
x_{\mathrm{OverallCond}}
}{2},
\\
s_{\mathrm{modernization}}(x)
&=
\frac{
x_{\mathrm{YearBuilt}}
+
x_{\mathrm{YearRemodAdd}}
}{2},
\\
s_{\mathrm{garage}}(x)
&=
\frac{
x_{\mathrm{GarageCars}}
+
x_{\mathrm{GarageArea}}
}{2}.
\end{align}

Larger values of each scorer indicate greater expected property value. The
four ordered calibration dimensions and their endpoint values are shown in
Table~\ref{tab:ames_calibration}. Within each dimension, all associated
features are varied jointly across five linearly spaced levels. Their Cartesian
product produces \(5^4 = 625\) synthetic calibration points.

\begin{table*}[t]
\centering
\begin{tabular}{lll}
\toprule
\textbf{Calibration dimension} & \textbf{Feature} & \textbf{Level 1--Level 5} \\
\midrule
Physical scale
& GrLivArea & 600--3500 \\
& TotalBsmtSF & 300--1800 \\
\midrule
Quality and condition
& OverallQual & 2--10 \\
& OverallCond & 3--9 \\
\midrule
Age and modernization
& YearBuilt & 1920--2010 \\
& YearRemodAdd & 1950--2010 \\
\midrule
Garage capacity
& GarageCars & 0--4 \\
& GarageArea & 0--1000 \\
\bottomrule
\end{tabular}
\caption{Calibration dimensions and endpoint values used for Ames Housing.}
\label{tab:ames_calibration}
\end{table*}

The \texttt{SalePrice} target is removed before scorer construction. Missing
numerical values are replaced using feature medians. Although categorical
features are ordinally encoded during preprocessing, the four scorers use only
the numerical features listed above.

\subsection{Breast Cancer Wisconsin}

The latent concept is malignancy risk. The original dataset target is
reoriented so that malignant cases are assigned one and benign cases are
assigned zero. This target is withheld during supervision construction.

Four heuristic scorers are defined as subset-specific means:
\begin{equation}
\label{eq:bcw_scorers}
s_g(x)
=
\frac{1}{|\mathcal{F}_g|}
\sum_{i \in \mathcal{F}_g} x_i,
\qquad
g \in
\{
\text{size},
\text{shape},
\text{texture},
\text{worst}
\},
\end{equation}
where each feature set \(\mathcal{F}_g\) comprises the corresponding variables
listed in Table~\ref{tab:bcw_calibration}.

Larger values indicate greater estimated malignancy risk. As in the Ames
experiment, the features associated with each calibration dimension are varied
jointly across five linearly spaced levels, yielding \(5^4 = 625\) synthetic
reference points.

\begin{table*}[t]
\centering
\begin{tabular}{lll}
\toprule
\textbf{Calibration dimension} & \textbf{Feature} & \textbf{Level 1--Level 5} \\
\midrule
Size
& Mean radius & 7--28 \\
& Mean perimeter & 45--190 \\
& Mean area & 150--2500 \\
\midrule
Shape irregularity
& Mean compactness & 0.02--0.35 \\
& Mean concavity & 0--0.45 \\
& Mean concave points & 0--0.20 \\
\midrule
Texture and smoothness
& Mean texture & 9--40 \\
& Mean smoothness & 0.05--0.16 \\
& Mean symmetry & 0.10--0.30 \\
\midrule
Worst-case measurements
& Worst radius & 8--36 \\
& Worst perimeter & 50--250 \\
& Worst area & 200--4300 \\
& Worst concavity & 0--1.25 \\
& Worst concave points & 0--0.30 \\
\bottomrule
\end{tabular}
\caption{Calibration dimensions and endpoint values used for Breast Cancer Wisconsin.}
\label{tab:bcw_calibration}
\end{table*}

\subsection{Wine Quality}
\label{app:wine_quality}

The latent concept is wine quality. Four heuristic scorers are defined using
monotonic physicochemical properties:
\begin{align}
s_{\text{acidity}}(x)
&=
\frac{
-x_{\text{volatile acidity}}
+
x_{\text{citric acid}}
}{2},
\\
s_{\text{additives}}(x)
&=
\frac{
-x_{\text{chlorides}}
+
x_{\text{sulphates}}
}{2},
\\
s_{\text{preservatives}}(x)
&=
\frac{
-x_{\text{free SO}_2}
-
x_{\text{total SO}_2}
}{2},
\\
s_{\text{profile}}(x)
&=
\frac{
x_{\text{alcohol}}
-
x_{\text{density}}
}{2}.
\end{align}

Larger values indicate higher expected quality. Each of the four corresponding
calibration dimensions is varied jointly across five linearly spaced levels,
yielding \(5^4 = 625\) synthetic calibration points. The calibration endpoints
are shown in Table~\ref{tab:wine_calibration}. Target quality labels are
removed prior to supervision construction and used solely for post hoc
evaluation.

\begin{table*}[t]
\centering
\begin{tabular}{lll}
\toprule
\textbf{Calibration dimension} & \textbf{Feature} & \textbf{Level 1--Level 5} \\
\midrule
Acidity
& Volatile acidity & 1.58--0.12 \\
& Citric acid & 0.0--1.0 \\
\midrule
Additives
& Chlorides & 0.611--0.012 \\
& Sulphates & 0.33--2.0 \\
\midrule
Preservatives
& Free sulfur dioxide & 72.0--1.0 \\
& Total sulfur dioxide & 289.0--6.0 \\
\midrule
Profile
& Alcohol & 8.0--14.9 \\
& Density & 1.00369--0.99007 \\
\bottomrule
\end{tabular}
\caption{Calibration dimensions and endpoint values used for Wine Quality.}
\label{tab:wine_calibration}
\end{table*}

\subsection{Calibration, Fusion, and Baselines}

For each scorer \(j\), the reference mean and population standard deviation
are computed over the corresponding 625 synthetic scorer outputs. Observed
scores are standardized using these reference statistics and averaged with
equal weight to obtain \(\kappa\).

The uncalibrated baseline is the direct arithmetic mean of the four raw scorer
outputs. We additionally compare against empirical \(z\)-score averaging,
Borda rank averaging, and the first principal component obtained using
unsupervised PCA. Each raw scorer is also evaluated individually.

Ames Housing and Wine Quality are evaluated primarily using Spearman
correlation, with Kendall correlation and top-10\% overlap reported as
secondary and additional metrics. Breast Cancer Wisconsin is evaluated
primarily using AUROC, with average precision and Spearman correlation with
malignancy status reported as secondary and additional metrics. Primary
results and paired comparisons are reported in
Table~\ref{tab:cross_domain}; the complete secondary and additional results
are reported in Table~\ref{tab:cross_domain_secondary}.

\begin{table*}[t]
\centering
\small
\begin{tabular}{llcc}
\toprule
\textbf{Dataset}
& \textbf{Method}
& \textbf{Secondary}
& \textbf{Additional} \\
\midrule

Ames Housing
& Calibrated Fusion ($\kappa$)
& 0.7197
& 73.97\% \\
& Empirical $z$-Score Avg.
& 0.7444
& \textbf{78.08\%} \\
& Borda Rank Averaging
& 0.7359
& 73.29\% \\
& Unsupervised PCA (PC1)
& \textbf{0.7456}
& \textbf{78.08\%} \\
& Uncalibrated Average
& 0.6890
& 73.29\% \\
& Best Single Scorer
& 0.6365
& 68.49\% \\

\midrule

Breast Cancer
& Calibrated Fusion ($\kappa$)
& 0.9770
& 0.8111 \\
& Empirical $z$-Score Avg.
& 0.9743
& 0.8084 \\
& Borda Rank Averaging
& 0.9774
& 0.8108 \\
& Unsupervised PCA (PC1)
& \textbf{0.9812}
& \textbf{0.8151} \\
& Uncalibrated Average
& 0.9456
& 0.7662 \\
& Best Single Scorer
& 0.9616
& 0.7880 \\

\midrule

Wine Quality
& Calibrated Fusion ($\kappa$)
& \textbf{0.4403}
& 44.38\% \\
& Empirical $z$-Score Avg.
& 0.4336
& 42.50\% \\
& Borda Rank Averaging
& 0.4339
& \textbf{46.25\%} \\
& Unsupervised PCA (PC1)
& 0.4349
& 38.75\% \\
& Uncalibrated Average
& 0.1511
& 20.62\% \\
& Best Single Scorer
& 0.3718
& 34.38\% \\

\bottomrule
\end{tabular}
\caption{
Secondary and additional cross-domain validation metrics.
For Ames Housing and Wine Quality, the secondary metric is Kendall's
\(\tau\) and the additional metric is top-10\% overlap.
For Breast Cancer Wisconsin, the secondary metric is average precision
and the additional metric is Spearman correlation with malignancy status.
Bold values indicate the highest observed value within each dataset and metric.
}
\label{tab:cross_domain_secondary}
\end{table*}

Finally, Table~\ref{tab:scorer_performance} reports the primary-metric
performance of each individual subset-specific scorer.

\begin{table*}[t]
\centering
\begin{tabular}{lllc}
\toprule
\textbf{Dataset}
& \textbf{Subset-Specific Scorer}
& \textbf{Metric}
& \textbf{Value} \\
\midrule
Ames Housing
& Physical scale
& Spearman $\rho$
& 0.8150 \\
& Age and modernization
& & 0.6814 \\
& Garage capacity
& & 0.6496 \\
& Quality and condition
& & 0.6197 \\

\midrule

Breast Cancer
& Worst-case features
& AUROC
& 0.9705 \\
Wisconsin
& Size
& & 0.9391 \\
& Shape irregularity
& & 0.9325 \\
& Texture and smoothness
& & 0.7775 \\

\midrule

Wine Quality
& Profile
& Spearman $\rho$
& 0.4757 \\
& Additives
& & 0.3911 \\
& Acidity
& & 0.3244 \\
& Preservatives
& & 0.1724 \\

\bottomrule
\end{tabular}
\caption{
Primary validation performance of each subset-specific scorer across all
three evaluation datasets.
}
\label{tab:scorer_performance}
\end{table*}

\section{Synthetic Homes Implementation Details}
\label{app:synthetic_homes}
This appendix describes the implementation of the proposed framework for the Synthetic Homes case study presented in Section~\ref{sec:exp:synth}.

\subsection{Calibration Features and Subset-Specific Scorers}

The latent concept is residential retrofit priority. We characterize this concept using four calibration features associated with two building-system categories: HVAC heating COP, HVAC cooling COP, roof R-value, and wall R-value. Increasing HVAC COP and insulation R-value correspond to increasing building efficiency and, consequently, decreasing retrofit need.

The implementation employs four subset-specific scorers associated with two feature subsets. Two text-derived scorers are produced by an LLM that evaluates a single home inspection note and returns separate estimates of HVAC and insulation retrofit need. Two structured scorers are heuristic functions operating on the corresponding numerical feature subsets. The structured HVAC scorer jointly evaluates heating and cooling COP, while the structured insulation scorer jointly evaluates roof and wall R-values. We denote the four raw scorer outputs by

\begin{align*}
    s_{\mathrm{text,HVAC}}, \\
    s_{\mathrm{text,insulation}}, \\
    s_{\mathrm{structured,HVAC}}, \\
    s_{\mathrm{structured,insulation}}.
\end{align*}

All four outputs are oriented so that larger values indicate greater retrofit need.

\subsection{Synthetic Ordinal Reference Space}

Each calibration feature is assigned five ordered efficiency levels, shown in Table~\ref{tab:sec:framework:calibration:bins}. Their Cartesian product yields \( 5^4=625 \) calibration points.

\begin{table}[h]
\centering
\small
\begin{tabular}{c|ccccc}
\toprule
\textbf{Calibration Feature} &
\multicolumn{5}{c}{\textbf{Increasing Efficiency} \(\longrightarrow\)} \\
\midrule
HVAC heating COP & 0.7 & 0.8 & 0.9 & 0.95 & 1.0 \\
HVAC cooling COP & 1.0 & 2.0 & 3.0 & 3.5 & 4.0 \\
Roof R-value     & 10  & 20  & 30  & 40   & 50  \\
Wall R-value     & 4   & 7   & 13  & 20   & 30  \\
\bottomrule
\end{tabular}
\caption{Calibration levels used to construct the Synthetic Homes ordinal reference space.}
\label{tab:sec:framework:calibration:bins}
\end{table}

The calibration values were selected using domain knowledge rather than estimated from the empirical distribution of the Synthetic Homes dataset. Their purpose is to define an interpretable reference space spanning clearly inefficient, intermediate, and highly efficient conditions while preserving the expected monotonic relationship between each feature and retrofit need. They should not be interpreted as building-code thresholds, normative retrofit recommendations, or representative quantiles of the existing housing stock. The values are held fixed across the experiments to provide a common reference distribution for estimating scorer-specific calibration statistics.

\subsection{Whole-Home Inspection Notes}

For each calibration point, the heating and cooling COP levels are collapsed into a single HVAC note level. Similarly, the wall and roof R-value levels are collapsed into a single insulation note level. For two constituent levels \( a, b \in \{ 1, \ldots ,5 \} \), the corresponding note level is

\begin{equation}
\label{eq:collapsed_note_level}
\ell(a,b)
=
\left\lfloor
\frac{a+b}{2}
+
\frac{1}{2}
\right\rfloor,
\end{equation}

which applies half-up rounding to their mean. For example, constituent levels 2 and 3 produce note level 3.

One HVAC component and one insulation component are selected using the resulting levels and concatenated to form a single whole-home inspection note. The text components used in the experiments are shown in Table~\ref{tab:inspection_note_templates}.

\begin{table*}[t]
\centering
\small
\begin{tabular}{c p{0.38\linewidth} p{0.38\linewidth}}
\toprule
\textbf{Level}
& \textbf{HVAC Component}
& \textbf{Insulation Component} \\
\midrule

1 &
The home has an aging central HVAC system showing substantial exterior corrosion and poor operating efficiency. &
Attic insulation is minimal, with exposed joists visible throughout, causing significant heat loss. \\

2 &
The HVAC system appears to be in working condition but is an older standard-efficiency model. &
No signs of added insulation were observed in the basement ceiling, suggesting potential energy inefficiency. \\

3 &
The home has a functioning central HVAC system of moderate age and efficiency, with no major defects observed but no recent high-efficiency upgrades. &
Walls appear to be adequately insulated based on construction year, though no upgrades were observed. \\

4 &
The HVAC system was recently replaced with a standard high-efficiency model and is expected to operate efficiently. &
Blown-in insulation is present in the attic to a depth of approximately 10 inches, providing good thermal resistance. \\

5 &
A state-of-the-art HVAC system with smart thermostats and variable-speed compressors was recently installed, maximizing energy efficiency. &
High-performance spray foam insulation was installed throughout the walls, attic, and basement, providing maximum energy efficiency. \\

\bottomrule
\end{tabular}
\caption{Text components used to construct home inspection notes at each ordinal efficiency level.}
\label{tab:inspection_note_templates}
\end{table*}

The complete note presents the insulation component followed by the HVAC component. The LLM receives this home note in a single request and returns two scores in the interval \( [0,1] \): one for HVAC retrofit need and one for insulation retrofit need. The model is instructed to evaluate the two categories separately using all relevant evidence in the complete note.

Because each text component depends only on a collapsed concept-level condition, multiple points in the \( 625 \)-point structured reference space produce the same textual realization. The five HVAC conditions and five insulation conditions yield \( 5^2=25 \) distinct home notes. LLM outputs are cached by the complete note, so each distinct note is evaluated once and its two outputs are reused for every calibration point that produces the same note. The repeated outputs nevertheless retain their multiplicity in the \( 625 \)-point calibration distribution.

\subsection{Structured Scorers}

For a numerical feature \( q \) with ordered calibration anchors

\[
v_1 < v_2 < \cdots < v_5,
\]

the implementation maps \( q \) onto an efficiency level \( e(q)\in[1,5] \). Values at or below \( v_1 \) are assigned efficiency level 1, and values at or above \( v_5 \) are assigned efficiency level 5. Between adjacent anchors \( v_i \) and \( v_{i+1} \), the efficiency level is obtained by linear interpolation:

\begin{equation}
\label{eq:feature_efficiency_level}
e(q)
=
i+
\frac{q-v_i}{v_{i+1}-v_i},
\qquad
v_i\leq q\leq v_{i+1}.
\end{equation}

The corresponding feature-specific retrofit-need value is

\begin{equation}
\label{eq:feature_retrofit_need}
r(q)=6-e(q),
\end{equation}

so that increasing efficiency produces decreasing retrofit need.

The structured HVAC score is the mean of the heating- and cooling-specific retrofit-need values:

\begin{equation}
\label{eq:structured_hvac}
s_{\mathrm{structured,HVAC}}(x)
=
\frac{
r(q_{\mathrm{heat}}(x))
+
r(q_{\mathrm{cool}}(x))
}{2}.
\end{equation}

The structured insulation score is defined analogously:

\begin{equation}
\label{eq:structured_insulation}
s_{\mathrm{structured,insulation}}(x)
=
\frac{
r(q_{\mathrm{roof}}(x))
+
r(q_{\mathrm{wall}}(x))
}{2}.
\end{equation}

At the five calibration anchors, each feature-specific retrofit-need value is respectively \( 5,4,3,2,1 \). The interpolation permits the same scorers to operate on values that fall between the predefined calibration anchors. No EnergyPlus energy-consumption output or held-out retrofit label is used to construct these structured scores.

\subsection{Text Scorer}

The text scorer uses GPT-4.1-mini with temperature set to zero. It receives the complete home inspection note and is instructed to return a JSON object containing

\[
s_{\mathrm{text,HVAC}}\in[0,1]
\qquad\text{and}\qquad
s_{\mathrm{text,insulation}}\in[0,1].
\]

A score of zero indicates that the corresponding retrofit is judged definitely unnecessary, while a score of one indicates that it is judged definitely necessary. Intermediate values represent partial need or uncertainty. The response is parsed and validated to ensure that both required fields are present and fall within the specified interval.

\subsection{Calibration and Fusion}

Each scorer is evaluated over all \( 625 \) calibration points. For subset-specific scorer \( j \), the implementation computes the mean

\[
\mu_j
=
\frac{1}{625}
\sum_{x\in\mathcal{S}}s_j(x)
\]

and population standard deviation

\[
\sigma_j
=
\sqrt{
\frac{1}{625}
\sum_{x\in\mathcal{S}}
\left(s_j(x)-\mu_j\right)^2
}.
\]

A scorer output for an evaluated instance is standardized as

\begin{equation}
\hat{s}_j(x)
=
\frac{s_j(x)-\mu_j}{\sigma_j}.
\end{equation}

The text-derived and structured outputs are first fused within each building-system category:

\begin{align}
\label{eq:fused_subsets}
\kappa_{\mathrm{HVAC}}(x) &= \tfrac{1}{2}\big(\hat{s}_{\mathrm{text,HVAC}}(x) + \hat{s}_{\mathrm{structured,HVAC}}(x)\big), \\
\kappa_{\mathrm{insul}}(x) &= \tfrac{1}{2}\big(\hat{s}_{\mathrm{text,insul}}(x) + \hat{s}_{\mathrm{structured,insul}}(x)\big).
\end{align}
The overall supervision signal is the equal-weight mean across all four calibrated scorers:
\begin{equation}
\label{eq:fused_synthetic_homes}
\begin{split}
\kappa(x) = \tfrac{1}{4} \big( & \hat{s}_{\mathrm{text,HVAC}}(x) + \hat{s}_{\mathrm{structured,HVAC}}(x) \\
& + \hat{s}_{\mathrm{text,insul}}(x) + \hat{s}_{\mathrm{structured,insul}}(x) \big),
\end{split}
\end{equation}
which equivalently evaluates to $\kappa(x) = \frac{1}{2}\big(\kappa_{\mathrm{HVAC}}(x) + \kappa_{\mathrm{insul}}(x)\big)$.

Larger values of \(\kappa_{\mathrm{HVAC}}\), \(\kappa_{\mathrm{insulation}}\), and \(\kappa\) indicate greater estimated retrofit need. Because these quantities are formed from standardized outputs, negative values indicate retrofit need below the relevant calibration-space mean rather than negative physical retrofit need.

The same note-construction procedure and scoring functions used to estimate the calibration statistics are applied during inference. Each runtime scorer output is therefore standardized using the reference distribution generated by that same scorer.

\section{Exploratory Psychometric Analysis}
\label{app:psychometrics}
To further investigate potential connections between calibration-first supervision construction and latent-variable measurement, we conducted an exploratory analysis on several psychometric datasets. Unlike the primary validation experiments in Section~\ref{sec:exp:cross_domain}, these datasets do not contain naturally heterogeneous evidence sources operating on incompatible scales. Instead, they consist of multiple subscales designed to measure related psychological or cognitive constructs. Consequently, these experiments should be interpreted as exploratory evidence rather than as direct validation of the proposed framework.

Our goal was to evaluate whether calibrated fusion can recover latent structure extracted directly from item-level responses. For each dataset, domain or subscale scores were treated as evidence sources and fused using the procedure described in Section~\ref{sec:framework}. The resulting fused score $\kappa$ was then compared against the first principal component (PC1) computed from the complete item-response matrix.

\subsection{Datasets}

We considered three datasets, all sourced from the Open Psychometrics data repository~\citep{OpenPsychometrics}.

\paragraph{BIG5.}
The BIG5 personality inventory contains 50 items organized into five domains: Extraversion, Neuroticism, Agreeableness, Conscientiousness, and Openness. Domain scores were treated as evidence sources and fused into a single score $\kappa$. The BIG5 dataset consists of responses to the IPIP Big Five Factor Markers questionnaire~\citep{goldbergDevelopmentMarkersBigFive1992}.

\paragraph{HEXACO.}
The HEXACO personality inventory~\citep{leePsychometricPropertiesHEXACO2004} contains 240 items organized into six domains: Honesty–Humility, Emotionality, Extraversion, Agreeableness, Conscientiousness, and Openness. Domain scores were again treated as evidence sources and fused into $\kappa$.

\paragraph{MGKT.}
The Multifactor General Knowledge Test (MGKT)~\citep{OpenPsychometricsMGKT, KirkegaardJensen2025MGKT} contains objective knowledge questions spanning multiple domains of general knowledge. Domain-level scores were treated as evidence sources and fused into $\kappa$.

For each dataset, the fused score was compared against the first principal component extracted from the underlying item-level response matrix.

\subsection{Results}

Table~\ref{tab:psychometric_validation} summarizes the results.

\begin{table*}[t]
\centering
\begin{tabular}{lcc}
\toprule
Dataset & Method & Spearman $\rho$ with Item-Level PC1 \\
\midrule
BIG5 & Calibrated Fusion ($\kappa$) & 0.615 \\
HEXACO & Calibrated Fusion ($\kappa$) & 0.252 \\
MGKT & Calibrated Fusion ($\kappa$) & 0.987 \\
\bottomrule
\end{tabular}
\caption{Agreement between calibrated fusion and the dominant latent factor extracted from item-level responses.}
\label{tab:psychometric_validation}
\end{table*}

Several observations emerge. First, MGKT exhibits an extremely strong correspondence between $\kappa$ and the dominant latent factor extracted from item responses, with a Spearman correlation of $0.987$. This suggests that calibrated fusion almost perfectly reproduces the ranking implied by item-level latent structure.

Second, BIG5 demonstrates a moderate relationship between $\kappa$ and PC1. This is unsurprising because the Big Five framework was designed to represent several distinct personality dimensions rather than a single latent trait.

Third, HEXACO exhibits substantially weaker agreement. HEXACO was explicitly developed as a multidimensional personality framework containing conceptually distinct domains, and therefore a single fused score is not expected to align strongly with a one-dimensional latent representation.

\subsection{Discussion}

These findings are consistent with established psychometric theory. Instruments characterized by a strong dominant latent factor should exhibit stronger agreement between fused scores and latent-variable estimates. Conversely, instruments designed to measure multiple relatively independent constructs should exhibit weaker agreement because no single latent dimension adequately summarizes the entire instrument.

An important observation is that calibrated fusion and uncalibrated averaging produced nearly identical results across all three psychometric datasets. Unlike the Ames Housing and Breast Cancer Wisconsin experiments, the psychometric evidence sources were already measured on highly comparable scales and used similar response formats. As a result, little scale mismatch existed between evidence sources, leaving relatively little work for the calibration procedure to perform.

This behavior should not be interpreted as a failure of calibration. Rather, it demonstrates that the method behaves sensibly when sources are already aligned and that calibration becomes most important when heterogeneous evidence sources operate on substantially different scales.

Overall, these experiments provide exploratory evidence that calibration-first fusion is connected to latent-variable modeling. In particular, the MGKT results suggest that domain-level fusion can recover latent structure similar to that obtained from item-level analysis. However, because the psychometric evidence sources were already expressed on comparable scales, these experiments do not directly evaluate the primary contribution of this paper—calibration across heterogeneous information sources—and are therefore presented as supplementary evidence rather than as primary validation.

\end{document}